\documentclass[10pt]{article}

\usepackage[T1]{fontenc}        % good PDF fonts (safe for arXiv/pdflatex)
\usepackage[utf8]{inputenc}     % utf-8 input (harmless on arXiv)
\usepackage{lmodern}            % Latin Modern fonts
\usepackage{microtype}          % better kerning & spacing

\usepackage[english]{babel}

\usepackage{amsmath, amssymb, amsthm, mathtools, bm}
\allowdisplaybreaks            % allow multi-page align/eqn arrays
\usepackage{natbib}            % author-year citations (plainnat, abbrvnat, etc.)

\usepackage{graphicx}
\usepackage{array}
\usepackage{booktabs}
\usepackage{multirow}

\usepackage{enumitem}
\usepackage[margin=1in]{geometry}

\usepackage{hyperref}
\hypersetup{
  colorlinks=true,
  linkcolor=blue,
  citecolor=blue,
  urlcolor=blue,
  pdfauthor={Neta Shoham and Haim Avron},
  pdftitle={<Your Title Here>}
}

\usepackage[capitalize,nameinlink]{cleveref}

\numberwithin{equation}{section}
\theoremstyle{plain}
\newtheorem{theorem}{Theorem}[section]
\newtheorem{thm}[theorem]{Theorem}

\theoremstyle{definition}

\theoremstyle{remark}

\author{%
  Neta Shoham\thanks{School of Mathematical Sciences, Tel Aviv University, Tel Aviv, Israel.
  \href{mailto:shohamne@mail.tau.ac.il}{shohamne@mail.tau.ac.il}} \and
  Haim Avron\thanks{School of Mathematical Sciences, Tel Aviv University, Tel Aviv, Israel.
  \href{mailto:haimav@tauex.tau.ac.il}{haimav@tauex.tau.ac.il}}}

\date{}  % no date on title

\begin{document}
\selectlanguage{english}%
\global\long\def\R{\mathbb{R}}%

\global\long\def\C{\mathbb{C}}%

\global\long\def\N{\mathbb{N}}%

\global\long\def\e{{\mathbf{e}}}%

\global\long\def\et#1{{\e(#1)}}%

\global\long\def\ef{{\mathbf{\et{\cdot}}}}%

\global\long\def\x{{\mathbf{x}}}%

\global\long\def\xt#1{{\x(#1)}}%

\global\long\def\xf{{\mathbf{\xt{\cdot}}}}%

\global\long\def\a{{\mathbf{a}}}%

\global\long\def\b{{\mathbf{b}}}%

\global\long\def\d{{\mathbf{d}}}%

\global\long\def\w{{\mathbf{w}}}%

\global\long\def\b{{\mathbf{b}}}%

\global\long\def\u{{\mathbf{u}}}%

\global\long\def\y{{\mathbf{y}}}%

\global\long\def\n{{\mathbf{n}}}%

\global\long\def\k{{\mathbf{k}}}%

\global\long\def\yt#1{{\y(#1)}}%

\global\long\def\yf{{\mathbf{\yt{\cdot}}}}%

\global\long\def\z{{\mathbf{z}}}%

\global\long\def\v{{\mathbf{v}}}%

\global\long\def\h{{\mathbf{h}}}%

\global\long\def\q{{\mathbf{q}}}%

\global\long\def\s{{\mathbf{s}}}%

\global\long\def\p{{\mathbf{p}}}%

\global\long\def\f{{\mathbf{f}}}%

\global\long\def\rb{{\mathbf{r}}}%

\global\long\def\rt#1{{\rb(#1)}}%

\global\long\def\rf{{\mathbf{\rt{\cdot}}}}%

\global\long\def\mat#1{{\ensuremath{\bm{\mathrm{#1}}}}}%

\global\long\def\vec#1{{\ensuremath{\bm{\mathrm{#1}}}}}%

\global\long\def\matN{\ensuremath{{\bm{\mathrm{N}}}}}%

\global\long\def\matX{\ensuremath{{\bm{\mathrm{X}}}}}%

\global\long\def\X{\ensuremath{{\bm{\mathrm{X}}}}}%

\global\long\def\matK{\ensuremath{{\bm{\mathrm{K}}}}}%

\global\long\def\K{\ensuremath{{\bm{\mathrm{K}}}}}%

\global\long\def\matA{\ensuremath{{\bm{\mathrm{A}}}}}%

\global\long\def\A{\ensuremath{{\bm{\mathrm{A}}}}}%

\global\long\def\matB{\ensuremath{{\bm{\mathrm{B}}}}}%

\global\long\def\B{\ensuremath{{\bm{\mathrm{B}}}}}%

\global\long\def\matC{\ensuremath{{\bm{\mathrm{C}}}}}%

\global\long\def\C{\ensuremath{{\bm{\mathrm{C}}}}}%

\global\long\def\matD{\ensuremath{{\bm{\mathrm{D}}}}}%

\global\long\def\D{\ensuremath{{\bm{\mathrm{D}}}}}%

\global\long\def\matE{\ensuremath{{\bm{\mathrm{E}}}}}%

\global\long\def\E{\ensuremath{{\bm{\mathrm{E}}}}}%

\global\long\def\matF{\ensuremath{{\bm{\mathrm{F}}}}}%

\global\long\def\F{\ensuremath{{\bm{\mathrm{F}}}}}%

\global\long\def\matH{\ensuremath{{\bm{\mathrm{H}}}}}%

\global\long\def\H{\ensuremath{{\bm{\mathrm{H}}}}}%

\global\long\def\matP{\ensuremath{{\bm{\mathrm{P}}}}}%

\global\long\def\P{\ensuremath{{\bm{\mathrm{P}}}}}%

\global\long\def\matU{\ensuremath{{\bm{\mathrm{U}}}}}%

\global\long\def\matV{\ensuremath{{\bm{\mathrm{V}}}}}%

\global\long\def\V{\ensuremath{{\bm{\mathrm{V}}}}}%

\global\long\def\matW{\ensuremath{{\bm{\mathrm{W}}}}}%

\global\long\def\matM{\ensuremath{{\bm{\mathrm{M}}}}}%

\global\long\def\M{\ensuremath{{\bm{\mathrm{M}}}}}%

\global\long\def\calA{{\cal A}}%

\global\long\def\calE{{\cal E}}%

\global\long\def\calF{{\cal F}}%

\global\long\def\calK{{\cal K}}%

\global\long\def\calH{{\cal H}}%

\global\long\def\calY{{\cal Y}}%

\global\long\def\calP{{\cal P}}%

\global\long\def\calX{{\cal X}}%

\global\long\def\calS{{\cal S}}%

\global\long\def\calT{{\cal T}}%

\global\long\def\Normal{{\cal \mathcal{N}}}%

\global\long\def\GP{{\cal \mathcal{GP}}}%

\global\long\def\matQ{{\mat Q}}%

\global\long\def\Q{{\mat Q}}%

\global\long\def\matR{\mat R}%

\global\long\def\matS{\mat S}%

\global\long\def\matY{\mat Y}%

\global\long\def\matI{\mat I}%

\global\long\def\I{\mat I}%

\global\long\def\matJ{\mat J}%

\global\long\def\matZ{\mat Z}%

\global\long\def\Z{\mat Z}%

\global\long\def\matW{{\mat W}}%

\global\long\def\W{{\mat W}}%

\global\long\def\matL{\mat L}%

\global\long\def\S#1{{\mathbb{S}_{N}[#1]}}%

\global\long\def\IS#1{{\mathbb{S}_{N}^{-1!}[#1]}}%

\global\long\def\PN{\mathbb{P}_{N}}%

\global\long\def\TNormS#1{\|#1\|_{2}^{2}}%

\global\long\def\ITNormS#1{\|#1\|_{2}^{-2}}%

\global\long\def\ONorm#1{\|#1\Vert_{1}}%

\global\long\def\TNorm#1{\|#1\|_{2}}%

\global\long\def\InfNorm#1{\|#1\|_{\infty}}%

\global\long\def\FNorm#1{\|#1\|_{F}}%

\global\long\def\FNormS#1{\|#1\|_{F}^{2}}%

\global\long\def\UNorm#1{\|#1\|_{\matU}}%

\global\long\def\UNormS#1{\|#1\|_{\matU}^{2}}%

\global\long\def\UINormS#1{\|#1\|_{\matU^{-1}}^{2}}%

\global\long\def\ANorm#1{\|#1\|_{\matA}}%

\global\long\def\BNorm#1{\|#1\|_{\mat B}}%

\global\long\def\ANormS#1{\|#1\|_{\matA}^{2}}%

\global\long\def\AINormS#1{\|#1\|_{\matA^{-1}}^{2}}%

\global\long\def\T{\textsc{T}}%

\global\long\def\conj{\textsc{*}}%

\global\long\def\pinv{\textsc{+}}%

\global\long\def\Prob{\operatorname{Pr}}%

\global\long\def\Expect{\operatorname{\mathbb{E}}}%

\global\long\def\ExpectC#1#2{{\mathbb{E}}_{#1}\left[#2\right]}%

\global\long\def\VarC#1#2{{\mathbb{\mathrm{Var}}}_{#1}\left[#2\right]}%

\global\long\def\dotprod#1#2#3{(#1,#2)_{#3}}%

\global\long\def\dotprodN#1#2{(#1,#2)_{{\cal N}}}%

\global\long\def\dotprodH#1#2{\langle#1,#2\rangle_{{\cal {\cal H}}}}%

\global\long\def\dotprodsqr#1#2#3{(#1,#2)_{#3}^{2}}%

\global\long\def\Trace#1{{\bf Tr}\left(#1\right)}%

\global\long\def\Vec#1{{\bf Vec}\left(#1\right)}%

\global\long\def\nnz#1{{\bf nnz}\left(#1\right)}%

\global\long\def\MSE#1{{\bf MSE}\left(#1\right)}%

\global\long\def\WMSE#1{{\bf WMSE}\left(#1\right)}%

\global\long\def\EWMSE#1{{\bf EWMSE}\left(#1\right)}%

\global\long\def\nicehalf{\nicefrac{1}{2}}%

\global\long\def\argmin{\operatornamewithlimits{argmin}}%

\global\long\def\argmax{\operatornamewithlimits{argmax}}%

\global\long\def\norm#1{\Vert#1\Vert}%

\global\long\def\sign{\operatorname{sign}}%

\global\long\def\proj{\operatorname{proj}}%

\global\long\def\diag{\operatorname{diag}}%

\global\long\def\VOPT{\operatorname\{VOPT\}}%

\global\long\def\dist{\operatorname{dist}}%

\global\long\def\diag{\operatorname{diag}}%

\global\long\def\supp{\operatorname{supp}}%

\global\long\def\sp{\operatorname{span}}%

\global\long\def\rank{\operatorname{rank}}%

\global\long\def\onehot{\operatorname{onehot}}%

\global\long\def\softmax{\operatorname{softmax}}%

\newcommand*\diff{\mathop{}\!\mathrm{d}} 

\global\long\def\dd{\diff}%

\global\long\def\whatlambda{\w_{\lambda}}%

\global\long\def\Plambda{\mat P_{\lambda}}%

\global\long\def\Pperplambda{\left(\mat I-\Plambda\right)}%

\global\long\def\Mlambda{\matM_{\lambda}}%

\global\long\def\Mlambdafull{\matM+\lambda\matI}%

\global\long\def\Mlambdafullinv{\left(\matM+\lambda\matI\right)^{-1}}%

\global\long\def\Mdaggerlambda{{\mathbf{\mat M_{\lambda}^{+}}}}%

\global\long\def\Xdaggerlambda{{\mathbf{\mat X_{\lambda}^{+}}}}%

\global\long\def\XT{{\mathbf{X}^{\T}}}%

\global\long\def\XXT{{\matX\mat X^{\T}}}%

\global\long\def\XTX{{\matX^{\T}\mat X}}%

\global\long\def\VT{{\mathbf{V}^{\T}}}%

\global\long\def\VVT{{\matV\mat V^{\T}}}%

\global\long\def\VTV{{\matV^{\T}\mat V}}%

\global\long\def\varphibar#1#2{{\bar{\varphi}_{#1,#2}}}%

\global\long\def\varphilambda{{\varphi_{\lambda}}}%

\title{\selectlanguage{english}%
Unbiased Stochastic Optimization for Gaussian Processes on Finite
Dimensional RKHS}
\maketitle
\begin{abstract}
Current methods for stochastic hyperparameter learning in {\em Gaussian Processes} (GPs)
 rely on approximations,
such as computing biased stochastic gradients or using inducing points
in stochastic variational inference. However, when using such methods we are not
guaranteed to converge to a stationary point of the true marginal
likelihood. In this work, we propose algorithms for exact stochastic
inference of GPs with kernels that induce a Reproducing Kernel Hilbert
Space (RKHS) of moderate finite dimension. Our
approach can also be extended to infinite dimensional  RKHSs at the cost of forgoing exactness. Both for finite and infinite dimensional RKHSs, our method achieves better
experimental results than existing methods when memory resources limit
the feasible batch size and the possible number of inducing points. 
\end{abstract}

\section{Introduction}

Gaussian Processes (GPs) provide a powerful probabilistic framework
which has been applied to a wide range of learning applications, such
as multi-task learning \citep{alvarez2012kernels,liu2018remarks},
active learning \citep{liu2018survey}, semi-supervised learning \cite{jean2018semi},
and reinforcement learning \citep{srinivas2010gaussian,shahriari2015taking}.
These successes can be attributed to the natural way in which uncertainty
is incorporated into the predictions via a Bayesian interpretation.
However, hyperparameter learning scales poorly. For a general covariance
function (kernel function in kernel method's parlance) the computational
cost grows as $O(n^{3})$ where $n$ denotes the number of samples,
and the storage resources grow as $O(n^{2})$. As a consequence,
approximations are required for any modern application that involves
GPs and big data.

It is therefore unsurprising that much research effort has
been devoted to approximate methods for learning GPs. For a comprehensive overview of
this research, we refer the reader to the work of \citet{liu2020gaussian}.
The most expensive operation in training GPs is the inversion of the
$n\times n$ kernel matrix which is required during the maximization
of the marginal likelihood and for the computation of the posterior
of the responses. As a result, much research has been devoted to providing
a cheap approximation of the kernel matrix. 

Broadly speaking approximations of GPs can be split in to {\em data dependent} and {\em data independent} methods.
Data dependent approximations usually come
with valuable probabilistic interpretations. An important work in
this line of research is the seminal work of \citet{quinonero2005unifying},
which provides a unified view on previous work by using the concept
of inducing points. Since then quite a few followup works have used
inducing points, while exploiting the tool of variational inference
as a theoretical platform \citep{titsias2009variational,nguyen2014collaborative,wilson2015kernel,zhao2016variational}.
This approach is closely related to the Nyström approximation of kernel
matrices \citep{zhao2016variational}.

In contrast, data independent methods rely on approximating
the kernel function itself. Typically, it is approximated by an inner product between low dimensional
vectors \citep{RahimiRect07,yang2015carte,shustin2021gauss}.

Approximating the inverse of the kernel matrix is enough for frequentist
kernel methods such as Kernel Ridge Regression. However, in order
to harness the full power of Bayesian kernel methods such as {\em Gaussian
Process Regression} (GPR), we need to be also able to maximize the marginal
likelihood, and for that an approximation of the inverse of the kernel
matrix is not enough, due to an additional log-determinant term. This
is especially true in cases where the covariance function depends
on a large number of parameters, e.g., evaluating the covariance function
involves a forward pass of a deep neural network \citep{wilson2016deep,calandra2016manifold,wilson2016stochastic}.
In such cases, it is important to be able to use a stochastic mini-batches
based optimization, since otherwise the cost of making a pass over
the entire dataset typically proves too expensive. Moreover, it is
well appreciated in the literature that when the model involves deep
neural networks, e.g., when the covariance function is defined by a
neural network, it is important to use stochastic gradients, since
they enable more efficient optimization and better generalization
\citep{goodfellow2016deep}.

Currently, two main approaches are used for stochastic hyperparameter learning in GPs. The first relies on the inducing‑points framework and employs {\em stochastic variational inference} (SVI)
\citep{hoffman2013stochastic,hensman2013gaussian,hoang2015unifying,wilson2016stochastic}.
The second approach is more direct: it is based on computing stochastic batches while
ignoring the fact that they provide only biased estimates of gradients.
Interestingly, a recent work shows that despite the bias, given a large enough batch size, the direct approach produces
almost optimal models (in terms of the marginal likelihood) \citep{chen2020stochastic}.

Both the SVI approach and the direct approach suffer from  several  disadvantages.
For example, consider the case in which the covariance function is
the inner product between two feature maps of a moderate dimension.
This can be either because this decomposition serves as an approximation
to another covariance function with an infinite dimensional \emph{
reproducing kernel Hilbert space} (RKHS), or because it has been defined
this way to begin with, e.g., as an inner product between features
that are created by passing the data through a neural network. Either
way, using the SVI approach amounts to imposing an additional unnecessary
approximation that comes from the need to use inducing points. As
for the second approach, its applicability is highly dependent on our ability to use large batches. This can be a serious impediment
in several cases, such as optimization on weak edge devices \citep{chen2016training}.

In this paper, we propose two stochastic optimization algorithms based
on mini-batches for maximizing the marginal likelihood of GPs (i.e., learning the hyperparameters). The
first algorithm is based on reframing the problem as a nonconvex-concave
minimax problem. We then leverage recent advancements in the theory
of solving such problems \citep{boct2020alternating,lin2020icml,luo2020neurips}
to propose a concrete algorithm. The second algorithm is based on
writing the loss function in compositional form, i.e., as the composition
of a function and the expected value of another function. We then
use the recently introduced Stochastic Compositional Gradient Descent
method~\citep{wang2017stochastic}. Our novel algorithms are applicable
for covariance functions connected with an RKHS of a moderate dimension,
and guarantee convergence of the marginal likelihood to a local minimum
for any batch size without need for further approximations. In the
infinite dimensional case (e.g., Gaussian covariance function) one can use our method on top of a low rank
approximation of the covariance function, e.g., using the random features method \cite{RahimiRect07}.
Our experiments show that not only is our method superior to
existing methods for stochastic optimization of the marginal likelihood
in the finite dimensional case when the batches have a moderated size, it is also
superior to the existing methods in the infinite dimensional case if the restriction
on the batch size is more severe, although in the infinite dimensional case the models found by our method are no longer optimal.

\paragraph{Additional Related Work}

Apart from the extensive literature on scaling GPs using approximations,
several works are focused on exact inference using sophisticated
distributed algorithms \citep{nguyen2019exact,wang2019exact}. In this
context, the biased stochastic gradient proposed by \citet{chen2020stochastic}
can be considered as an economical method for exact inference, given
that the covariance function and the system enable computation in
large enough batches. 

\section{Preliminaries}

\subsection{Notations and Basic Definitions}

For a function $f$:${\cal U}\to\R$ and $U=\left(u_{1},\dots,u_{m}\right)\in{\cal U}^{m}$,
$\f=f\left(U\right)$ is a vector in $\R^{m}$ such that $\f_{i}=f\left(u_{i}\right)$.
Similarly, for a binary function $k:{\cal U}\times{\cal U}\to\R$,
$K=k\left(U,U\right)$ is a matrix in $\R^{m\times m}$ such that
$K_{ij}=k\left(u_{i},u_{j}\right)$. Given a size $m$ vector $\b$
and ${\cal S}=\left(s_{1},\dots,s_{p}\right)\in\left\{ 1,\dots,m\right\} ^{p}$,
we use $\b_{{\cal S}}$ to denote the size $p$ vector such that the
$i$'th coordinate of $\b_{{\cal S}}$ is equal to the $s_{i}$'th coordinate
of $\b$. In a similar way if $C$ is an $m\times n$ matrix then
$C_{{\cal S}}$ is a $p\times n$ matrix such that the $i$'th row of
$C_{S}$ is equal to the $s_{i}$'th row of $C$. Finally, if ${\cal R}=\left(r_{1},\dots,r_{q}\right)\in\left\{ 1,\dots,n\right\} ^{q}$
then $C_{{\cal SR}}$ is a $p\times q$ matrix such that the $\left(i,j\right)$
coordinate of $C_{{\cal SR}}$ is equal to the $\left(s_{i},r_{j}\right)$
coordinate of $C$.

For a square matrix $A$ we use the notation $\left|A\right|$ to
denote the determinant of $A$. If $\calS$ is a finite sequence or
a set then $\left|\calS\right|$ denotes its length or size. Whenever
we use $\left\langle A,B\right\rangle $ where $A$ and $B$ are real
matrices, $\left\langle \cdot,\cdot\right\rangle $ symbolizes the
Frobenius inner product which is defined as 
\[
\left\langle A,B\right\rangle =\Trace{AB^{T}}.
\]
If $A$ is a real matrix then $\left\Vert A\right\Vert $ is the Frobenius
norm of $A$ so, 
\[
\left\Vert A\right\Vert =\sqrt{\left\langle A,A\right\rangle }.
\]
For a vector $\v$, $\left\Vert \v\right\Vert $ is the Euclidean norm of $\v\in\R^{q}$. 

For any closed convex set $\Omega\subseteq\R^{q}$ and $\v\in\R^{q}$,
$\proj_{\Omega}\left(\v\right)=\arg\min_{\v'\in \Omega}\left\Vert v'-v\right\Vert$ is the Euclidean projection of $\v$
on $\Omega$. 

\subsection{Hyperparameter Learning in Gaussian Process Regression}

Let $\calX$ be a feature space and let $k_{\alpha}:\calX\times\calX\to\R$
be a positive definite covariance function parameterized by hyperparameters
$\alpha\in\R^{m}$. For any $\alpha \in \R^m$, let $f_{\alpha}$ be a random function on $\calX$
distributed as a zero mean GP whose covariance is $k_{\alpha}$, that
is, for any $j\in\N$, $U\in\calX^{j}$: 
\[
f_{\alpha}\left(U\right)\sim\Normal\left(0,k_{\alpha}\left(U,U\right)\right).
\]
Let $X\in\calX^{n}$, $\y\in\R^{n}$ be a training set, where $n$
is the number of training samples. In {\em Gaussian Process Regression} (GPR), it is assumed that $\y$ is a sample
of the random vector $f_{\bar{\alpha}}\left(X\right)+\bar{\sigma}^{2}\mathbf{\epsilon}$
where $\mathbf{\epsilon}\sim\Normal\left(0,I_{n}\right)$, and $\bar{\alpha},\bar{\sigma}^{2}$
are the true hyperparameters of the model.

In GPR, hyperparameters are learned by solving a maximum
likelihood type II problem, i.e., maximizing the marginal likelihood.
\[
p\left(\y|X,\alpha,\sigma\right)=\Normal\left(\y|0,K\left(\alpha\right)+\sigma^{2}I\right)
\]
with respect to $\alpha$ and $\sigma$, where $K\left(\alpha\right)=k_{\alpha}\left(X,X\right)$.
The maximization of the marginal likelihood $p\left(\y|X,\alpha,\sigma\right)$
is equivalent to the minimization of 
\[
\y^{T}\left(K\left(\alpha\right)+\sigma^{2}I\right)^{-1}\y+\log\left|K\left(\alpha\right)+\sigma^{2}I\right|.
\]
See \citet{rasmussen2003gaussian} for details. 
In this work we further assume that the covariance function has the following form: 
\[
k_{\alpha}\left(x,x'\right)=\phi_{\alpha}\left(x\right)^{T}\phi_{\alpha}\left(x'\right)
\]
 for some feature map $\phi_{\alpha}:\calX\to\R^{d}$. It can be shown
that for all $\lambda>0$, $V\in\R^{n\times d}$ and $\b\in\R^{d}$
we have that 
\[
\b^{T}\left(VV^{T}+\lambda I\right)^{-1}\b=\min_{\w}\frac{1}{\lambda}\left\Vert V\w-\b\right\Vert +\left\Vert \w\right\Vert ^{2}
\]
(see Appendix \ref{sec:Missing-Proofs}). As a result, maximizing $p\left(\y|X,\alpha,\sigma\right)$
is equivalent to minimizing 
\[
l\left(\theta\right)=\frac{1}{\sigma^{2}}\left\Vert Z\left(\alpha\right)\w-\y\right\Vert ^{2}+\left\Vert \w\right\Vert ^{2}+\log\left|F\left(\theta\right)\right|+\left(n-d\right)\log\sigma^{2},
\]
where 
\begin{align*}
\theta & =\left(\w,\alpha,\sigma^{2}\right),\\
Z\left(\alpha\right) & =\left(\phi_{\alpha}\left(x_{1}\right),\dots,\phi_{\alpha}\left(x_{n}\right)\right)^{T}\in\R^{n\times d},\\
F\left(\theta\right) & =Z\left(\alpha\right)^{T}Z\left(\alpha\right)+\sigma^{2}I\in\R^{d\times d}.
\end{align*}

The goal of this work is to propose algorithms for minimizing $l\left(\theta\right)$ using a stochastic
gradient method based on mini-batches. This would have been straightforward
if we could write 
\[
\nabla_{\theta}\log\left|F\left(\theta\right)\right|=\sum_{i=1}^{n}G\left(\theta;x_{i}\right)
\]
for some function $G(\cdot;\cdot)$. However, there is no obvious
decomposition of this form.

\section{Stochastic Optimization for Gaussian Processes}

This section contains our main contribution: two novel stochastic
mini-batched based algorithms for minimizing $l(\theta)$. Before
detailing our approaches, we make few additional notations: 
\begin{align*}
g_{i}\left(\theta\right) & =\frac{1}{\sigma^{2}}\left(\phi_{\alpha}\left(x_{i}\right)\w-y_{i}\right)^{2}+\frac{1}{n}\left\Vert \w\right\Vert ^{2}+\frac{1}{n}\left(n-d\right)\log\sigma^{2}\\
g\left(\theta\right) & =\sum_{i=1}^{n}g_{i}\left(\theta\right)\\
F_{i}\left(\theta\right) & =\phi_{\alpha}\left(x_{i}\right)\phi_{\alpha}\left(x_{i}\right)^{T}+\frac{1}{n}\sigma^{2}I\\
h\left(A\right) & =\log\left|A\right|
\end{align*}
So we can write, 
\begin{align*}
F\left(\theta\right) & =\sum_{i=1}^{n}F_{i}\left(\theta\right)\\
l\left(\theta\right) & =g\left(\theta\right)+h\left(F\left(\theta\right)\right).
\end{align*}

\subsection{A Minimax Approach}\label{subsec:A-Minimax-Approach}

Our first step is to replace the problem 
\[
\min_{\theta}g\left(\theta\right)+h\left(F\left(\theta\right)\right)
\]
with the equivalent problem 
\begin{gather}
\min_{\theta,A}g\left(\theta\right)+h\left(A\right)\nonumber \\
\text{s.t. }A=F\left(\theta\right)\label{eq:constraint}
\end{gather}
The next step is to replace the last problem with a parameterized
problem, such that the hard constraint is replaced with a penalty
term, and the penalty term is driven to infinity. To do so, let us
first define the optimization problems
\[
\min_{\zeta}l_{\mu}\left(\zeta\right)
\]
where $\zeta=\left(\theta,A\right)$ and
\[
l_{\mu}\left(\zeta\right)=g\left(\theta\right)+h\left(A\right)+\mu\frac{\left\Vert A-F\left(\theta\right)\right\Vert }{\left\Vert A\right\Vert }.
\]
Now, suppose that Problem~(\ref{eq:constraint}) admits an optimal solution. 
Let the sequences $\{\mu_k\}$ and $\{\zeta_k\}$ satisfy $\mu_k \to \infty$ and, for each $k$, let $\zeta_k$ minimize $l_{\mu_k}(\zeta)$. 
If $\zeta^{\star}$ is an accumulation point of $\{\zeta_k\}$, then $\zeta^{\star}$ is an optimal solution of Problem~(\ref{eq:constraint}) \cite[Theorem 6.6]{ruszczynski2006nonlinear}.

The reason we use 
\[
\frac{\left\Vert A-F\left(\theta\right)\right\Vert }{\left\Vert A\right\Vert }
\]
 instead of just $\left\Vert A-F\left(\theta\right)\right\Vert $
is to avoid finding critical points around $F\left(\theta\right)=0$.

Unfortunately, the term $\left\Vert A-F\left(\theta\right)\right\Vert $
exhibits the same issues as $h\left(F\left(\theta\right)\right)=\log\det F(\theta)$
and does not allow straightforward unbiased mini-batch based stochastic
gradients. 

In order to handle this, we use the fact that for an element $v$
in Euclidean space we have that $\left\Vert v\right\Vert =\max_{\left\Vert u\right\Vert \le1}\left\langle u,v\right\rangle $.
Thus, minimizing $l_{\mu}\left(\zeta\right)$ is equivalent to 
\begin{gather}
\min_{\zeta}\max_{\left\Vert B\right\Vert \le1}\Psi\left(\zeta,B\right),\label{eq:min-max}
\end{gather}
where 
\[
\Psi\left(\zeta,B\right)=g\left(\theta\right)+h\left(A\right)+\mu\frac{\left\langle B,A-F\left(\theta\right)\right\rangle }{\left\Vert A\right\Vert }
\]
Now, we can write
\begin{equation}
\Psi\left(\zeta,B\right)=\sum^n_{i=1}\psi\left(\zeta,B;x_{i}\right)\label{eq:sum_psi}
\end{equation}
where 
\[
\psi\left(\zeta,B;x_{i},\mu\right)=g_{i}\left(\theta\right)+\frac{1}{n}h\left(A\right)+\mu\frac{\left\langle B,\frac{1}{n}A-F_{i}\left(\theta\right)\right\rangle }{\left\Vert A\right\Vert }
\]
Thus, Problem~(\ref{eq:min-max}) naturally admits unbiased stochastic
gradients based on mini-batches taken separately for the minimization
and the maximization parts.

For better stability, we restrict $A$ to be positive semidefinite
such that $A\succeq\sigma^{2}I$, and restrict $\sigma\ge\sigma_{\min}$, where $\sigma_min$ is a non-learned hyperparameter.
To further restrict the search space, we constrain $\zeta\preceq\zeta_{\max}$, where $\zeta_{\max}$ is another non-learned hyperparameter.
Together, the constraints on $\zeta$ can be represented  by $\zeta\in\Omega_{1}$,
where $\Omega_{1}$ is a convex set.

Let $\Omega_{2}=\left\{ B\in\R^{d\times d}\mid\left\Vert B\right\Vert \le1\right\} $.
The algorithm we propose is based on the following update rule for
solving Problem~(\ref{eq:min-max}):
\begin{align*}
1. & \ \zeta_{t+1}\ =\proj_{\Omega_{1}}\left(\ensuremath{\zeta_{k}-}a\nabla_{\zeta}\frac{n}{s}\sum_{i\in\calS_{t+1}}\psi\left(\zeta_{t},B_{t};x_{i},\mu_{t}\right)\right)\\
2. & \ B_{t+1}=\proj_{\Omega_{2}}\left(\ensuremath{B_{t}+}\ensuremath{b}\nabla_{B}\frac{n}{s}\sum_{i\in\bar{\calS}_{t+1}}\psi\left(\zeta_{t+1},B_{t};x_{i},\mu_{t}\right)\right)
\end{align*}
where $\calS_{t}$ and $\bar{\calS}_{t}$ are sets of $s$ indices randomly
chosen from $\left\{ 1,\dots,n\right\} $ independently from all previous
iterations. Recent work by \cite{boct2020alternating} shows that
in case of a minimax problem where maximization is of a concave function
over a convex constraint\footnote{The setup in \cite{boct2020alternating} is more general.},
then with a few additional mild conditions on the objective function,
an algorithm based on the above update rule converges to a stationary
point in $O\left(\epsilon^{-8}\right)$ iterations. 

\subsection{Stochastic Compositional Gradient Descent Approach}\label{subsec:Stochastic-Compositional-Gradien}

Consider a loss function $l:\Theta\to\R$ of the form $l\left(\theta\right)=v\left(u\left(\theta\right)\right)$
where $u:\Theta\to\R^{p}$ and $v:\R^{p}\to\R$ are differentiable
functions, and assume that $u\left(\theta\right)=\Expect_{\omega}\tilde{u}\left(\theta;\omega\right)$
for a differentiable function $\tilde{u}\left(\theta;\omega\right)$
that depends on a random variable $\omega$.\emph{ Stochastic Compositional
Gradient Descent} (SCGD) \citep{wang2017stochastic} is  an intuitive
algorithm that alternates between two steps: updating the solution
$\theta_{t}$ by a stochastic gradient iteration, and estimating $u\left(\theta_{t}\right)$
using an iterative weighted average of past values. More precisely,
the update rule of of SCGD is given by:
\begin{align*}
1. & \ \theta_{t+1}=\ensuremath{\theta_{t}-}a_{t}\left\langle \nabla v\left(\eta_{t}\right),\nabla\tilde{u}\left(\theta;\omega_{t}\right)\right\rangle \\
2. & \ \eta_{t+1}=\left(1-b_{t}\right)\eta_{t}+b_{t}\tilde{u}\left(\theta;\omega_{t}\right).
\end{align*}
where $\omega_{1},\omega_{2},\dots$ are samples from $\omega$ in an
i.i.d. manner, and $a_0,a_1, \dots,b_0,b_1,\dots$ degrees of freedom in the algorithm. Under few additional standard conditions on $u,v,\tilde{u}$,
\cite{wang2017stochastic} showed a convergence rate of ${\cal O}\left(\epsilon^{-4}\right)$
if we choose $a_{t}=t^{-\frac{3}{4}}$ and $b_{t}=t^{-\frac{1}{2}}$,
see \citep[Theorem 8]{wang2017stochastic}. 

In order to use SCGD for our loss, we define: 
\begin{align*}
u\left(\theta\right) & =\left(g\left(\theta\right),F\left(\theta\right)\right),\\
v\left(u_{1},u_{2}\right) & =u_{1}+h\left(u_{2}\right),\\
\tilde{u}\left(\theta;\omega\right) & =\frac{n}{\left|\calS\right|}\sum_{i\in\calS}\left(g_{i}\left(\theta\right),F_{i}\left(\theta\right)\right)
\end{align*}
where $\omega=\calS$ is a random set of indices chosen from $\left\{ 1,\dots,n\right\} .$
With that, given the fact that 
\[
\frac{\partial\log\left|A\right|}{\partial A}=A^{-1}
\]
 we can write an explicit update rule for our SCGD-based algorithm: 
\begin{align*}
1. & \ \theta_{t+1}=\ensuremath{\theta_{t}-}a_{t}\sum_{i\in\calS}\nabla\left[g_{i}\left(\theta_{t}\right)+\left\langle \tilde{F}_{t}^{-1},F_{i}\left(\theta\right)\right\rangle \right]\\
2. & \ \tilde{F}_{t+1}=\left(1-b_{t}\right)\tilde{F}_{t}+b_{t}\frac{n}{\left|\calS\right|}\sum_{i\in{\cal S}}F_{i}\left(\theta_{t}\right).
\end{align*}

\section{Comparison to Existing Methods}

In this work we suggest two novel methods for stochastic optimization
of the marginal likelihood. We recognize two main competing methods
that also optimize the marginal likelihood stochastically. The simplest
among them is what we henceforth refer to as \emph{Biased Stochastic Gradient
Decent} (BSGD) \citep{chen2020stochastic}. The idea in BSGD is to take 
the gradient of the marginal likelihood using data only from the current
batch, ignoring the fact that this produces only a biased estimate of the
full gradient. On the other hand \emph{Scalable Variational Gaussian
Process} (SVGP)  of \citet{hensman2015scalable} is a sophisticated approach  
that approximates the original inference problem using inducing points
and stochastic variational inference. 

Unlike our algorithms, how well each of the competing methods approximates the solution of the true problem
depends on memory consumption. In the case of BSGD, the bias in stochastic gradients shrinks as batch size increases. On the other
hand, inference quality  SVGP  crucially depends on
the number of inducing points which need to be processed forward and
backward by a neural network at each iteration. 

\subsection{Complexity Analysis}

Let ${\rm C}$ and ${\rm M}$ be the computational and the memory
complexity in the computation of $\frac{\partial\phi_{\alpha}\left(x\right)}{\partial\alpha}$,
and let $b$ be the batch size. For SVGP, assume that the number
of the inducing points is in the same order of $b$. 

\paragraph{Number of inducing points vs.\ batch size.}
For the sake if comparison, We choose the number of inducing points to be roughly the same as the mini-batch size because each inducing point requires a forward and backward pass through the network, storing activations and gradients just like a data sample. Thus, both sets consume memory in a similar way during each iteration, and matching their sizes keeps the per-iteration memory cost balanced and comparable.

The computational complexity of computing one optimization iteration
of both BSGD and SVGP is $O\left(b^{3}\right)+O\left(b^{2}d\right)+b{\rm C}$
while both our algorithm take $O\left(d^{3}\right)+O\left(bd^{2}\right)+b{\rm C}$
per iteration. For storage, BSGD and SVGP take $O\left(b^{2}\right)+O\left(bd\right)+b{\rm M}$
while ours takes $O\left(d^{2}\right)+O\left(bd\right)+b{\rm M}$. 

Typically ${\rm M}$ is relatively large as it is the storage used
by AD to compute the gradient of a neural network. In small devices
that means one might have to use small batches. Our approach aims
not to harm the exactness of the algorithm in this case. Naturally, to use our algorithms effectively, the feature dimension $d=\dim\!\phi_{\alpha}(x)$ must remain moderate, since memory scales as $O(d^{2})$ and computation as $O(d^{3})$.

\begin{table*}
\begin{centering}
\caption{Complexity analysis for one optimization iteration: ${\rm C}$ and
${\rm M}$ are the computational and the memory complexity in the
computation of $\frac{\partial\phi_{\alpha}\left(x\right)}{\partial\alpha}$.
$b$ is the batch size.}\label{tab:Complexity-analysis}
\par\end{centering}
\centering{}%
\begin{tabular}{ccc}
Algorithm & Computations & Storage\tabularnewline
\midrule
SVGP & $O\left(b^{3}\right)+O\left(b^{2}d\right)+b{\rm C}$ & $O\left(b^{2}\right)+O\left(bd\right)+b{\rm M}$\tabularnewline
BSGD & $O\left(b^{3}\right)+O\left(b^{2}d\right)+b{\rm C}$ & $O\left(b^{2}\right)+O\left(bd\right)+b{\rm M}$\tabularnewline
Minimax & $O\left(d^{3}\right)+O\left(bd^{2}\right)+b{\rm C}$ & $O\left(d^{2}\right)+O\left(bd\right)+b{\rm M}$\tabularnewline
CSGD & $O\left(d^{3}\right)+O\left(bd^{2}\right)+b{\rm C}$ & $O\left(d^{2}\right)+O\left(bd\right)+b{\rm M}$\tabularnewline
\end{tabular}
\end{table*}

\subsection{Relation to BSGD}\label{subsec:A-Discussion-about}

Our approach is close to the BSGD approach of \citet{chen2020stochastic}. BSGD is based on the following update formula:
\[
\theta_{t+1}=\theta_{t}-a_{t}\nabla l\left(\theta_{t};{\cal S}_{t}\right)
\]
where ${\cal S}_{t}$ are an independent random batch of indices and
\[
l\left(\theta;{\cal S}\right)=\y_{{\cal S}}^{T}\left(K_{{\cal SS}}\left(\theta\right)+\sigma^{2}I\right)^{-1}\y_{S}+\log\left|K_{{\cal SS}}\left(\theta\right)+\sigma^{2}I\right|.
\]
In our case, since $k_{\alpha}\left(x\right)=\phi_{\alpha}\left(x\right)^{T}\phi_{\alpha}\left(x\right)$,
we can write $l\left(\theta;{\cal S}\right)$ as 
\[
l\left(\theta;{\cal S}\right)=\sum_{i}g_{i}\left(\theta\right)+h\left(\sum_{i}F_{i}\left(\theta\right)\right).
\]
This leads to the update formula 
\[
\theta_{t+1}=\ensuremath{\theta_{t}-}a_{t}\sum_{i\in\calS}\nabla\left[g_{i}\left(\theta_{t}\right)+\left\langle \tilde{F}_{t}^{-1},F_{i}\left(\theta\right)\right\rangle \right],\,\,\text{where }\tilde{F}_{t}=\sum_{i\in\calS}F_{i}\left(\theta_{t}\right)
\]
Writing the update formula of the BSGD algorithm this way emphasizes
the fact that the only difference between the BSGD and the SCGD algorithms
lays in the way in which $F\left(\theta_{t}\right)$ is approximated: using  $\tilde{F}_{t}=\sum_{i\in\calS}F_{i}\left(\theta_{t}\right)$
in BSGD, and $\tilde{F}_{t+1}=\left(1-b_{t}\right)\tilde{F}_{t}+b_{t}\frac{n}{\left|\calS\right|}\sum_{i\in{\cal S}}F_{i}\left(\theta_{t}\right)$
in SCGD. Intuitively, the exponential smoothing that occurs SCGD should
provide additional numerical robustness beyond the theoretical advantage
of converging to a stationary point without an error which depends
on the batch size. 

We can also see that in BSGD the scale of $\tilde{F}_{t}$
is incorrect because it misses the multiplication of $\sum_{i\in{\cal S}}F_{i}\left(\theta_{t}\right)$
by $\frac{n}{\left|\calS\right|}$.

\section{Experimental Result}

\begin{table*}
\caption{Negative log marginal likelihood with neural network and Linear kernel.}\label{tab:Negative-log-marginal-linear}

\begin{tabular}{llcccc} \toprule  & method & \textbf{MINIMAX(Ours)} & \textbf{SCGD(Ours)} & SVGP & BSGD \\ batch size & name &  &  &  &  \\ \midrule \multirow[t]{9}{*}{32} & bike & \textbf{-1.482$\pm$0.249} & -1.454$\pm$0.264 & -0.895$\pm$0.296 & -1.167$\pm$0.701 \\  & elevators & \textbf{-1.588$\pm$0.287} & -1.530$\pm$0.195 & -0.150$\pm$0.436 & -1.125$\pm$0.017 \\  & keggdirected & \textbf{-0.745$\pm$0.076} & -0.740$\pm$0.093 & -0.151$\pm$1.057 & -0.688$\pm$0.018 \\  & keggundirected & -0.737$\pm$0.004 & \textbf{-0.739$\pm$0.011} & -0.655$\pm$0.007 & -0.716$\pm$0.025 \\  & kin40k & -1.713$\pm$0.063 & \textbf{-1.760$\pm$0.090} & 0.792$\pm$0.881 & -0.684$\pm$0.028 \\  & pol & 2.128$\pm$0.097 & 2.016$\pm$0.126 & \textbf{1.643$\pm$0.108} & 2.689$\pm$0.175 \\  & protein & \textbf{0.392$\pm$0.060} & 0.393$\pm$0.053 & 0.848$\pm$0.218 & 0.640$\pm$0.016 \\  & slice & 0.436$\pm$0.836 & 0.500$\pm$0.855 & \textbf{-0.045$\pm$0.127} & 0.459$\pm$0.089 \\  & tamielectric & \textbf{0.178$\pm$0.001} & 0.179$\pm$0.001 & 0.179$\pm$0.001 & 0.179$\pm$0.001 \\ \cline{1-6} \multirow[t]{9}{*}{64} & bike & \textbf{-1.378$\pm$0.275} & -1.366$\pm$0.259 & -0.644$\pm$0.297 & -1.282$\pm$0.617 \\  & elevators & \textbf{-1.523$\pm$0.276} & -1.427$\pm$0.173 & -0.152$\pm$0.440 & -1.139$\pm$0.019 \\  & keggdirected & \textbf{-0.780$\pm$0.059} & -0.748$\pm$0.105 & -0.126$\pm$1.042 & -0.726$\pm$0.020 \\  & keggundirected & \textbf{-0.747$\pm$0.004} & \textbf{-0.747$\pm$0.005} & -0.670$\pm$0.005 & -0.710$\pm$0.015 \\  & kin40k & -1.654$\pm$0.045 & \textbf{-1.682$\pm$0.110} & 0.298$\pm$0.970 & -0.952$\pm$0.059 \\  & pol & 2.206$\pm$0.299 & 2.171$\pm$0.250 & \textbf{1.610$\pm$0.128} & 2.413$\pm$0.069 \\  & protein & \textbf{0.325$\pm$0.013} & 0.381$\pm$0.060 & 0.659$\pm$0.104 & 0.618$\pm$0.021 \\  & slice & 0.471$\pm$1.059 & 0.086$\pm$0.110 & \textbf{-0.162$\pm$0.164} & 0.304$\pm$0.143 \\  & tamielectric & \textbf{0.178$\pm$0.001} & \textbf{0.178$\pm$0.001} & 0.179$\pm$0.001 & 0.179$\pm$0.001 \\ \cline{1-6} \multirow[t]{9}{*}{128} & bike & -1.340$\pm$0.039 & \textbf{-1.358$\pm$0.029} & -0.832$\pm$0.231 & -1.311$\pm$0.126 \\  & elevators & \textbf{-1.678$\pm$0.018} & -1.399$\pm$0.022 & 0.044$\pm$0.003 & -1.173$\pm$0.005 \\  & keggdirected & \textbf{-0.805$\pm$0.042} & -0.778$\pm$0.091 & -0.189$\pm$1.078 & -0.763$\pm$0.009 \\  & keggundirected & -0.725$\pm$0.022 & \textbf{-0.750$\pm$0.007} & -0.679$\pm$0.007 & -0.735$\pm$0.013 \\  & kin40k & -1.590$\pm$0.024 & \textbf{-1.618$\pm$0.030} & 0.003$\pm$0.536 & -1.224$\pm$0.038 \\  & pol & 2.215$\pm$0.156 & 2.150$\pm$0.140 & \textbf{1.583$\pm$0.014} & 2.223$\pm$0.037 \\  & protein & \textbf{0.289$\pm$0.008} & 0.327$\pm$0.010 & 0.709$\pm$0.069 & 0.576$\pm$0.003 \\  & slice & 0.019$\pm$0.043 & 0.044$\pm$0.068 & \textbf{-0.197$\pm$0.024} & -0.038$\pm$0.119 \\  & tamielectric & \textbf{0.178$\pm$0.001} & \textbf{0.178$\pm$0.001} & 0.179$\pm$0.001 & 0.179$\pm$0.001 \\ \cline{1-6} \bottomrule \end{tabular} 
\end{table*}

\begin{table*}
\begin{centering}
\caption{Negative log marginal likelihood with neural network and Gaussian
kernel.}\label{tab:Negative-log-marginal-gaussian}
\par\end{centering}
\centering{}\begin{tabular}{llcccc} \toprule  & method & \textbf{MINIMAX(Ours)} & \textbf{SCGD(Ours)} & SVGP & BSGD \\ batch size & name &  &  &  &  \\ \midrule \multirow[t]{9}{*}{16} & bike & -1.013$\pm$0.082 & -1.005$\pm$0.040 & \textbf{-1.038$\pm$0.029} & 0.105$\pm$0.107 \\  & elevators & -1.094$\pm$0.018 & \textbf{-1.135$\pm$0.009} & -0.915$\pm$0.012 & -1.081$\pm$0.014 \\  & keggdirected & \textbf{-0.755$\pm$0.004} & -0.754$\pm$0.007 & -0.640$\pm$0.012 & -0.566$\pm$0.012 \\  & keggundirected & \textbf{-0.714$\pm$0.000} & \textbf{-0.714$\pm$0.000} & -0.641$\pm$0.000 & -0.615$\pm$0.000 \\  & kin40k & -1.559$\pm$0.015 & \textbf{-1.606$\pm$0.031} & -0.697$\pm$0.027 & -0.515$\pm$0.042 \\  & pol & 2.066$\pm$0.029 & 1.980$\pm$0.027 & \textbf{1.892$\pm$0.035} & 2.656$\pm$0.192 \\  & protein & 0.485$\pm$0.007 & \textbf{0.476$\pm$0.007} & 0.625$\pm$0.016 & 0.716$\pm$0.010 \\  & slice & 1.548$\pm$0.058 & 1.556$\pm$0.042 & \textbf{0.369$\pm$0.062} & 0.461$\pm$0.143 \\  & tamielectric & \textbf{0.179$\pm$0.001} & \textbf{0.179$\pm$0.001} & \textbf{0.179$\pm$0.001} & \textbf{0.179$\pm$0.001} \\ \cline{1-6} \multirow[t]{9}{*}{32} & bike & -0.929$\pm$0.023 & -0.968$\pm$0.065 & \textbf{-1.098$\pm$0.030} & 0.083$\pm$0.055 \\  & elevators & -1.042$\pm$0.023 & \textbf{-1.111$\pm$0.016} & -0.926$\pm$0.007 & -1.092$\pm$0.011 \\  & keggdirected & -0.764$\pm$0.007 & \textbf{-0.767$\pm$0.007} & -0.661$\pm$0.007 & -0.650$\pm$0.012 \\  & keggundirected & -0.735$\pm$0.000 & \textbf{-0.737$\pm$0.000} & -0.656$\pm$0.000 & -0.714$\pm$0.000 \\  & kin40k & \textbf{-1.585$\pm$0.018} & -1.571$\pm$0.017 & -0.796$\pm$0.031 & -0.695$\pm$0.021 \\  & pol & 2.115$\pm$0.027 & 2.112$\pm$0.029 & \textbf{1.807$\pm$0.023} & 2.521$\pm$0.192 \\  & protein & \textbf{0.479$\pm$0.008} & \textbf{0.479$\pm$0.004} & 0.595$\pm$0.006 & 0.674$\pm$0.009 \\  & slice & 1.950$\pm$0.120 & 1.815$\pm$0.142 & \textbf{0.159$\pm$0.042} & 0.248$\pm$0.098 \\  & tamielectric & \textbf{0.179$\pm$0.001} & \textbf{0.179$\pm$0.001} & \textbf{0.179$\pm$0.001} & \textbf{0.179$\pm$0.001} \\ \cline{1-6} \multirow[t]{9}{*}{64} & bike & -0.912$\pm$0.035 & -0.946$\pm$0.049 & \textbf{-1.153$\pm$0.011} & -0.114$\pm$0.060 \\  & elevators & -0.988$\pm$0.013 & -1.056$\pm$0.018 & -0.932$\pm$0.004 & \textbf{-1.111$\pm$0.007} \\  & keggdirected & -0.762$\pm$0.006 & \textbf{-0.771$\pm$0.009} & -0.672$\pm$0.005 & -0.694$\pm$0.008 \\  & keggundirected & -0.739$\pm$0.000 & \textbf{-0.743$\pm$0.000} & -0.666$\pm$0.000 & -0.691$\pm$0.000 \\  & kin40k & \textbf{-1.547$\pm$0.022} & -1.539$\pm$0.039 & -0.864$\pm$0.026 & -0.947$\pm$0.040 \\  & pol & 2.433$\pm$0.081 & 2.399$\pm$0.059 & \textbf{1.754$\pm$0.023} & 2.237$\pm$0.030 \\  & protein & 0.479$\pm$0.007 & \textbf{0.474$\pm$0.007} & 0.585$\pm$0.011 & 0.644$\pm$0.011 \\  & slice & 1.964$\pm$0.846 & 1.899$\pm$0.764 & 0.092$\pm$0.036 & \textbf{-0.018$\pm$0.059} \\  & tamielectric & \textbf{0.179$\pm$0.001} & \textbf{0.179$\pm$0.001} & \textbf{0.179$\pm$0.001} & \textbf{0.179$\pm$0.001} \\ \cline{1-6} \bottomrule \end{tabular} 
\end{table*}

In our experiments we consider a covariance function of the form 
\[
k_{\alpha}\left(x,x'\right)=k'_{u}\left(g_{w}\left(x\right),g_{w}\left(x'\right)\right),
\]
where $g_{w}$ is a neural network comprised of two fully connected
layers, both with output dimension of 128 and ReLU activation function,
and $k'_{u}$ which is either the linear kernel 
\[
k'_{u}\left(z,z'\right)=\left\langle z,z'\right\rangle 
\]
or the Gaussian kernel parametrized with two hyper parameter, length
scale $u_{1}$ and magnitude $u_{2}$, that is 
\[
k'_{u}\left(z,z'\right)=u_{2}e^{-\frac{\left\Vert z-z'\right\Vert ^{2}}{2u_{1}^{2}}}
\]

\paragraph*{Setup:}

The experiment is designed such that we can see the influence of the
batch size on the result. We used nine regression datasets from the UCI
repository \citep{asuncion2007uci}, all with number of samples above
14,000 and less the 60,000. We tested the algorithms using different
batch sizes: for the linear kernel we examined batch sizes of 32,
64, 128, 256 and 512, and for the Gaussian kernel we examined batch
sizes of 16, 32, 64, 128 and 256. For SVGP, we set the number of
inducing point equal to batch size. For our algorithms, labeled MINIMAX
(Subsection \ref{subsec:A-Minimax-Approach}) and  SCGD (Subsection
\ref{subsec:Stochastic-Compositional-Gradien}), we approximated
the Gaussian kernel $k'_{u}$ using the \emph{Random Fourier Features}
method with random features $\varphi$ of dimension 1000 such that
\[
k'_{u}\left(z,z'\right)\approx\left\langle \varphi\left(z\right),\varphi\left(z'\right)\right\rangle .
\]

We ran each test on five different splits of 90\% train 10\% test.
We used AdaDelta for all methods and for each combination of dataset,
split and method, and used grid search in order to select the learning
rate that achieves minimal marginal likelihood. For  SCGD,
we fixed $b_{t}=0.9$. Note that this is very close to $1$ which
means that SCGD becomes quite similar to BSGD and much of the improvement
comes just from the correct scaling of $\tilde{F}_{t}$ . For MINIMAX
we fixed $\mu_{t}=1.0$, since we found it enough for achieving good results
despite that in theory it should be increased in an outer loop. We
ran each algorithm for 100 epochs and used the hyperparameters from
the iteration in which the marginal likelihood achieved its minimal
value. 

\paragraph*{Result with linear kernel - the exact case:}

It seems that the fact that our algorithm does an exact stochastic
optimization brings a significant improvement over existing methods
in the optimization of the marginal likelihood. As expected, we can
see that this advantage is more significant when the batch size is
smaller (see Table \ref{tab:Negative-log-marginal-linear}). However,
lower negative log marginal likelihood does not always translate to
lower MSE on the test set as we can see in Table \ref{tab:RMSE-with-neural-linear}.
We saw that sometimes by using a suboptimal learning rate that does
not achieve the minimal loss our algorithms can achieve better results
in terms of test RMSE. However, since our work is focused on optimization,
we do not use procedures such as early stopping or cross validation which
could potentially improve the result from the test RMSE perspective. 

\paragraph*{Result with Gaussian kernel:}

Since here for our algorithms, MINIMAX and SCGD, we use an
approximated Gaussian kernel based random Fourier features, we are
no longer performing an exact optimization in this case. However,
we can see in Table \ref{tab:Negative-log-marginal-linear} that  
although we use an approximated kernel eventually when the restrictions
on batch size are high our methods do a better job than the existing  
methods in the optimization of the marginal likelihood. The advantage
of our methods in the optimization is also reflected in the test error
(see Table \ref{tab:RMSE-with-neural-gaussian}).

We see that in both cases, the finite dimensional RKHS and the infinite
dimensional RKHS, unlike the existing inference algorithms, there is
no degradation in the result of our algorithms with the decreasing
of the batch size. This property can be vital for inference on weak
edge devices where the memory restrictions limit the possible batch
size. 

\begin{table*}
\caption{RMSE with neural network and Linear kernel.}\label{tab:RMSE-with-neural-linear}

\begin{tabular}{llcccc} \toprule  & method & \textbf{MINIMAX(Ours)} & \textbf{SCGD(Ours)} & SVGP & BSGD \\ batch size & name &  &  &  &  \\ \midrule \multirow[t]{9}{*}{32} & bike & 0.099$\pm$0.011 & 0.095$\pm$0.007 & 0.093$\pm$0.055 & \textbf{0.082$\pm$0.073} \\  & elevators & 0.111$\pm$0.008 & 0.158$\pm$0.114 & 0.217$\pm$0.071 & \textbf{0.090$\pm$0.003} \\  & keggdirected & \textbf{0.120$\pm$0.011} & 0.123$\pm$0.015 & 0.376$\pm$0.557 & 0.122$\pm$0.009 \\  & keggundirected & 0.120$\pm$0.006 & 0.121$\pm$0.008 & \textbf{0.119$\pm$0.003} & \textbf{0.119$\pm$0.003} \\  & kin40k & 0.061$\pm$0.012 & \textbf{0.056$\pm$0.011} & 0.651$\pm$0.310 & 0.126$\pm$0.010 \\  & pol & \textbf{2.280$\pm$0.130} & 2.328$\pm$0.123 & 2.593$\pm$0.174 & 3.402$\pm$0.152 \\  & protein & \textbf{0.426$\pm$0.024} & 0.428$\pm$0.023 & 0.550$\pm$0.079 & 0.483$\pm$0.006 \\  & slice & 0.642$\pm$0.327 & 0.697$\pm$0.282 & \textbf{0.591$\pm$0.119} & 0.992$\pm$0.339 \\  & tamielectric & 0.290$\pm$0.002 & \textbf{0.289$\pm$0.002} & \textbf{0.289$\pm$0.002} & \textbf{0.289$\pm$0.002} \\ \cline{1-6} \multirow[t]{9}{*}{64} & bike & 0.112$\pm$0.012 & 0.108$\pm$0.012 & 0.127$\pm$0.054 & \textbf{0.078$\pm$0.064} \\  & elevators & 0.111$\pm$0.013 & 0.114$\pm$0.023 & 0.216$\pm$0.071 & \textbf{0.094$\pm$0.010} \\  & keggdirected & \textbf{0.118$\pm$0.008} & 0.120$\pm$0.011 & 0.377$\pm$0.557 & 0.119$\pm$0.009 \\  & keggundirected & \textbf{0.119$\pm$0.006} & 0.121$\pm$0.008 & 0.121$\pm$0.008 & 0.123$\pm$0.006 \\  & kin40k & 0.065$\pm$0.015 & \textbf{0.060$\pm$0.010} & 0.443$\pm$0.338 & 0.097$\pm$0.009 \\  & pol & 2.449$\pm$0.160 & \textbf{2.263$\pm$0.151} & 2.415$\pm$0.115 & 3.067$\pm$0.145 \\  & protein & 0.430$\pm$0.007 & \textbf{0.426$\pm$0.021} & 0.494$\pm$0.046 & 0.473$\pm$0.003 \\  & slice & \textbf{0.479$\pm$0.066} & 0.615$\pm$0.288 & 0.554$\pm$0.160 & 0.680$\pm$0.173 \\  & tamielectric & 0.290$\pm$0.002 & \textbf{0.289$\pm$0.002} & \textbf{0.289$\pm$0.002} & \textbf{0.289$\pm$0.002} \\ \cline{1-6} \multirow[t]{9}{*}{128} & bike & 0.125$\pm$0.005 & 0.123$\pm$0.011 & 0.111$\pm$0.024 & \textbf{0.056$\pm$0.004} \\  & elevators & 0.110$\pm$0.004 & 0.101$\pm$0.002 & 0.250$\pm$0.007 & \textbf{0.091$\pm$0.003} \\  & keggdirected & \textbf{0.116$\pm$0.008} & 0.118$\pm$0.011 & 0.370$\pm$0.561 & 0.123$\pm$0.018 \\  & keggundirected & 0.124$\pm$0.008 & 0.121$\pm$0.007 & \textbf{0.118$\pm$0.006} & 0.119$\pm$0.007 \\  & kin40k & 0.073$\pm$0.002 & \textbf{0.071$\pm$0.003} & 0.283$\pm$0.193 & 0.078$\pm$0.003 \\  & pol & 2.469$\pm$0.174 & \textbf{2.413$\pm$0.234} & 2.537$\pm$0.094 & 2.642$\pm$0.110 \\  & protein & \textbf{0.435$\pm$0.009} & 0.436$\pm$0.009 & 0.513$\pm$0.023 & 0.461$\pm$0.006 \\  & slice & \textbf{0.528$\pm$0.086} & 0.546$\pm$0.181 & 0.542$\pm$0.111 & 0.583$\pm$0.086 \\  & tamielectric & 0.290$\pm$0.002 & \textbf{0.289$\pm$0.002} & \textbf{0.289$\pm$0.002} & \textbf{0.289$\pm$0.002} \\ \cline{1-6} \bottomrule \end{tabular} 
\end{table*}

\begin{table*}
\caption{RMSE with neural network and Gaussian kernel.}\label{tab:RMSE-with-neural-gaussian}

\begin{tabular}{llcccc} \toprule  & method & \textbf{MINIMAX(Ours)} & \textbf{SCGD(Ours)} & SVGP & BSGD \\ batch size & name &  &  &  &  \\ \midrule \multirow[t]{9}{*}{16} & bike & 0.081$\pm$0.012 & 0.096$\pm$0.022 & \textbf{0.046$\pm$0.007} & 0.220$\pm$0.014 \\  & elevators & 0.096$\pm$0.003 & 0.098$\pm$0.004 & \textbf{0.089$\pm$0.002} & \textbf{0.089$\pm$0.003} \\  & keggdirected & \textbf{0.118$\pm$0.008} & \textbf{0.118$\pm$0.007} & 0.122$\pm$0.005 & 0.127$\pm$0.005 \\  & keggundirected & 0.131$\pm$0.000 & 0.128$\pm$0.000 & \textbf{0.116$\pm$0.000} & 0.138$\pm$0.000 \\  & kin40k & \textbf{0.039$\pm$0.002} & 0.040$\pm$0.002 & 0.110$\pm$0.007 & 0.147$\pm$0.007 \\  & pol & 2.212$\pm$0.146 & \textbf{2.200$\pm$0.104} & 2.674$\pm$0.081 & 3.774$\pm$0.440 \\  & protein & \textbf{0.387$\pm$0.009} & 0.388$\pm$0.009 & 0.471$\pm$0.008 & 0.507$\pm$0.012 \\  & slice & 0.519$\pm$0.149 & \textbf{0.496$\pm$0.110} & 0.562$\pm$0.067 & 0.661$\pm$0.090 \\  & tamielectric & \textbf{0.289$\pm$0.002} & \textbf{0.289$\pm$0.002} & \textbf{0.289$\pm$0.002} & \textbf{0.289$\pm$0.002} \\ \cline{1-6} \multirow[t]{9}{*}{32} & bike & 0.093$\pm$0.010 & 0.095$\pm$0.012 & \textbf{0.047$\pm$0.010} & 0.235$\pm$0.024 \\  & elevators & 0.100$\pm$0.004 & 0.098$\pm$0.004 & \textbf{0.088$\pm$0.003} & 0.089$\pm$0.002 \\  & keggdirected & \textbf{0.118$\pm$0.008} & \textbf{0.118$\pm$0.007} & 0.121$\pm$0.007 & 0.125$\pm$0.008 \\  & keggundirected & 0.130$\pm$0.000 & 0.128$\pm$0.000 & \textbf{0.115$\pm$0.000} & 0.116$\pm$0.000 \\  & kin40k & \textbf{0.039$\pm$0.002} & \textbf{0.039$\pm$0.001} & 0.096$\pm$0.005 & 0.124$\pm$0.003 \\  & pol & 2.210$\pm$0.133 & \textbf{2.209$\pm$0.080} & 2.552$\pm$0.117 & 3.338$\pm$0.151 \\  & protein & \textbf{0.387$\pm$0.008} & 0.389$\pm$0.010 & 0.461$\pm$0.009 & 0.493$\pm$0.007 \\  & slice & \textbf{0.511$\pm$0.039} & 0.515$\pm$0.065 & 0.527$\pm$0.139 & 0.565$\pm$0.102 \\  & tamielectric & \textbf{0.289$\pm$0.002} & \textbf{0.289$\pm$0.002} & \textbf{0.289$\pm$0.002} & \textbf{0.289$\pm$0.001} \\ \cline{1-6} \multirow[t]{9}{*}{64} & bike & 0.101$\pm$0.008 & 0.106$\pm$0.014 & \textbf{0.035$\pm$0.001} & 0.209$\pm$0.009 \\  & elevators & 0.089$\pm$0.003 & 0.098$\pm$0.007 & \textbf{0.088$\pm$0.003} & \textbf{0.088$\pm$0.003} \\  & keggdirected & 0.118$\pm$0.008 & \textbf{0.117$\pm$0.008} & 0.119$\pm$0.006 & 0.122$\pm$0.007 \\  & keggundirected & 0.128$\pm$0.000 & 0.127$\pm$0.000 & \textbf{0.113$\pm$0.000} & 0.117$\pm$0.000 \\  & kin40k & \textbf{0.041$\pm$0.001} & \textbf{0.041$\pm$0.002} & 0.088$\pm$0.005 & 0.096$\pm$0.003 \\  & pol & \textbf{2.348$\pm$0.077} & 2.353$\pm$0.088 & 2.512$\pm$0.081 & 3.020$\pm$0.071 \\  & protein & \textbf{0.388$\pm$0.008} & 0.389$\pm$0.005 & 0.456$\pm$0.005 & 0.481$\pm$0.004 \\  & slice & 0.616$\pm$0.163 & 0.546$\pm$0.129 & \textbf{0.507$\pm$0.116} & 0.512$\pm$0.120 \\  & tamielectric & \textbf{0.289$\pm$0.002} & \textbf{0.289$\pm$0.002} & \textbf{0.289$\pm$0.002} & \textbf{0.289$\pm$0.002} \\ \cline{1-6} \bottomrule \end{tabular}
\end{table*}

\section{Conclusion}

In many cases the covariance function of a GP is defined as an inner
product between features of a finite and moderate dimension. In this
case, the problem of minimizing the negative-log-marginal-likelihood
has a shape of a standard ridge regression problem with a non standard
regularization term in the form of the log-determinant of the covariance
matrix of the representations plus $\sigma^{2}I$. In this work we
developed two techniques that enable solving this problem with stochastic
mini-batches which unlike the exiting methods does not depend on large
batches in order to be exact. When the inference involves forward
and backward passes of a feed forward neural net this property is
of great importance and can be an enabler of such inference architectures
on weak edge devices.

\paragraph{Limitations:} We remark that in comparison to BSGD, both MINIMAX and SCGD are more complex and
require tuning of additional hyperparameters in order to achieve
the minimal negative log marginal likelihood. In addition, the optimization advantage is not fully reflected in the test error, so cross-validation is still needed to select the best algorithm.
\clearpage

\bibliographystyle{tmlr}

\bibliography{bibtex}

\begin{thebibliography}{34}
\providecommand{\natexlab}[1]{#1}
\providecommand{\url}[1]{\texttt{#1}}
\expandafter\ifx\csname urlstyle\endcsname\relax
  \providecommand{\doi}[1]{doi: #1}\else
  \providecommand{\doi}{doi: \begingroup \urlstyle{rm}\Url}\fi

\bibitem[Alvarez et~al.(2012)Alvarez, Rosasco, and Lawrence]{alvarez2012kernels}
MA~Alvarez, L~Rosasco, and ND~Lawrence.
\newblock Kernels for vector-valued functions: a review.
\newblock \emph{Foundations and Trends{\textregistered} in Machine Learning}, 4\penalty0 (3):\penalty0 195--266, 2012.

\bibitem[Asuncion \& Newman(2007)Asuncion and Newman]{asuncion2007uci}
Arthur Asuncion and David Newman.
\newblock Uci machine learning repository, 2007.

\bibitem[Bo{\c{t}} \& B{\"o}hm(2020)Bo{\c{t}} and B{\"o}hm]{boct2020alternating}
Radu~Ioan Bo{\c{t}} and Axel B{\"o}hm.
\newblock Alternating proximal-gradient steps for (stochastic) nonconvex-concave minimax problems.
\newblock \emph{arXiv preprint arXiv:2007.13605}, 2020.

\bibitem[Calandra et~al.(2016)Calandra, Peters, Rasmussen, and Deisenroth]{calandra2016manifold}
Roberto Calandra, Jan Peters, Carl~Edward Rasmussen, and Marc~Peter Deisenroth.
\newblock Manifold gaussian processes for regression.
\newblock In \emph{2016 International Joint Conference on Neural Networks (IJCNN)}, pp.\  3338--3345. IEEE, 2016.

\bibitem[Chen et~al.(2020)Chen, Zheng, Al~Kontar, and Raskutti]{chen2020stochastic}
Hao Chen, Lili Zheng, Raed Al~Kontar, and Garvesh Raskutti.
\newblock Stochastic gradient descent in correlated settings: A study on gaussian processes.
\newblock \emph{Advances in Neural Information Processing Systems}, 33, 2020.

\bibitem[Chen et~al.(2016)Chen, Xu, Zhang, and Guestrin]{chen2016training}
Tianqi Chen, Bing Xu, Chiyuan Zhang, and Carlos Guestrin.
\newblock Training deep nets with sublinear memory cost.
\newblock \emph{arXiv preprint arXiv:1604.06174}, 2016.

\bibitem[Goodfellow et~al.(2016)Goodfellow, Bengio, and Courville]{goodfellow2016deep}
Ian Goodfellow, Yoshua Bengio, and Aaron Courville.
\newblock \emph{Deep learning}.
\newblock MIT press, 2016.

\bibitem[Hensman et~al.(2013)Hensman, Fusi, and Lawrence]{hensman2013gaussian}
James Hensman, Nicol\`{o} Fusi, and Neil~D. Lawrence.
\newblock Gaussian processes for big data.
\newblock In \emph{Proceedings of the Twenty-Ninth Conference on Uncertainty in Artificial Intelligence}, UAI'13, pp.\  282--290, Arlington, Virginia, USA, 2013. AUAI Press.

\bibitem[Hensman et~al.(2015)Hensman, Matthews, and Ghahramani]{hensman2015scalable}
James Hensman, Alexander Matthews, and Zoubin Ghahramani.
\newblock Scalable variational gaussian process classification.
\newblock In \emph{Artificial Intelligence and Statistics}, pp.\  351--360. PMLR, 2015.

\bibitem[Hoang et~al.(2015)Hoang, Hoang, and Low]{hoang2015unifying}
Trong~Nghia Hoang, Quang~Minh Hoang, and Bryan Kian~Hsiang Low.
\newblock A unifying framework of anytime sparse gaussian process regression models with stochastic variational inference for big data.
\newblock In \emph{International Conference on Machine Learning}, pp.\  569--578. PMLR, 2015.

\bibitem[Hoffman et~al.(2013)Hoffman, Blei, Wang, and Paisley]{hoffman2013stochastic}
Matthew~D Hoffman, David~M Blei, Chong Wang, and John Paisley.
\newblock Stochastic variational inference.
\newblock \emph{Journal of Machine Learning Research}, 14\penalty0 (5), 2013.

\bibitem[Jean et~al.(2018)Jean, Xie, and Ermon]{jean2018semi}
Neal Jean, Sang~Michael Xie, and Stefano Ermon.
\newblock Semi-supervised deep kernel learning: Regression with unlabeled data by minimizing predictive variance.
\newblock In S.~Bengio, H.~Wallach, H.~Larochelle, K.~Grauman, N.~Cesa-Bianchi, and R.~Garnett (eds.), \emph{Advances in Neural Information Processing Systems}, volume~31. Curran Associates, Inc., 2018.
\newblock URL \url{https://proceedings.neurips.cc/paper/2018/file/9d28de8ff9bb6a3fa41fddfdc28f3bc1-Paper.pdf}.

\bibitem[Lin et~al.(2020)Lin, Jin, and Jordan]{lin2020icml}
Tianyi Lin, Chi Jin, and Michael Jordan.
\newblock On gradient descent ascent for nonconvex-concave minimax problems.
\newblock In Hal~Daum{\'e} III and Aarti Singh (eds.), \emph{Proceedings of the 37th International Conference on Machine Learning}, volume 119 of \emph{Proceedings of Machine Learning Research}, pp.\  6083--6093. PMLR, 13--18 Jul 2020.
\newblock URL \url{https://proceedings.mlr.press/v119/lin20a.html}.

\bibitem[Liu et~al.(2018{\natexlab{a}})Liu, Cai, and Ong]{liu2018remarks}
Haitao Liu, Jianfei Cai, and Yew-Soon Ong.
\newblock Remarks on multi-output gaussian process regression.
\newblock \emph{Knowledge-Based Systems}, 144:\penalty0 102--121, 2018{\natexlab{a}}.

\bibitem[Liu et~al.(2018{\natexlab{b}})Liu, Ong, and Cai]{liu2018survey}
Haitao Liu, Yew-Soon Ong, and Jianfei Cai.
\newblock A survey of adaptive sampling for global metamodeling in support of simulation-based complex engineering design.
\newblock \emph{Structural and Multidisciplinary Optimization}, 57\penalty0 (1):\penalty0 393--416, 2018{\natexlab{b}}.

\bibitem[Liu et~al.(2020)Liu, Ong, Shen, and Cai]{liu2020gaussian}
Haitao Liu, Yew-Soon Ong, Xiaobo Shen, and Jianfei Cai.
\newblock When gaussian process meets big data: A review of scalable gps.
\newblock \emph{IEEE transactions on neural networks and learning systems}, 31\penalty0 (11):\penalty0 4405--4423, 2020.

\bibitem[Luo et~al.(2020)Luo, Ye, Huang, and Zhang]{luo2020neurips}
Luo Luo, Haishan Ye, Zhichao Huang, and Tong Zhang.
\newblock Stochastic recursive gradient descent ascent for stochastic nonconvex-strongly-concave minimax problems.
\newblock In H.~Larochelle, M.~Ranzato, R.~Hadsell, M.~F. Balcan, and H.~Lin (eds.), \emph{Advances in Neural Information Processing Systems}, volume~33, pp.\  20566--20577. Curran Associates, Inc., 2020.
\newblock URL \url{https://proceedings.neurips.cc/paper/2020/file/ecb47fbb07a752413640f82a945530f8-Paper.pdf}.

\bibitem[Nguyen et~al.(2019)Nguyen, Filippone, and Michiardi]{nguyen2019exact}
Duc-Trung Nguyen, Maurizio Filippone, and Pietro Michiardi.
\newblock Exact gaussian process regression with distributed computations.
\newblock In \emph{Proceedings of the 34th ACM/SIGAPP Symposium on Applied Computing}, pp.\  1286--1295, 2019.

\bibitem[Nguyen et~al.(2014)Nguyen, Bonilla, et~al.]{nguyen2014collaborative}
Trung~V Nguyen, Edwin~V Bonilla, et~al.
\newblock Collaborative multi-output gaussian processes.
\newblock In \emph{UAI}, pp.\  643--652. Citeseer, 2014.

\bibitem[Quinonero-Candela \& Rasmussen(2005)Quinonero-Candela and Rasmussen]{quinonero2005unifying}
Joaquin Quinonero-Candela and Carl~Edward Rasmussen.
\newblock A unifying view of sparse approximate gaussian process regression.
\newblock \emph{The Journal of Machine Learning Research}, 6:\penalty0 1939--1959, 2005.

\bibitem[Rahimi \& Recht(2007)Rahimi and Recht]{RahimiRect07}
Ali Rahimi and Benjamin Recht.
\newblock Random features for large-scale kernel machines.
\newblock In \emph{Proceedings of the 21st International Conference on Neural Information Processing Systems}, NIPS'07, pp.\  1177–1184, Red Hook, NY, USA, 2007. Curran Associates Inc.
\newblock ISBN 9781605603520.

\bibitem[Rasmussen(2003)]{rasmussen2003gaussian}
Carl~Edward Rasmussen.
\newblock Gaussian processes in machine learning.
\newblock In \emph{Summer school on machine learning}, pp.\  63--71. Springer, 2003.

\bibitem[Ruszczynski(2006)]{ruszczynski2006nonlinear}
Andrzej Ruszczynski.
\newblock Nonlinear optimization, 2006.

\bibitem[Shahriari et~al.(2015)Shahriari, Swersky, Wang, Adams, and De~Freitas]{shahriari2015taking}
Bobak Shahriari, Kevin Swersky, Ziyu Wang, Ryan~P Adams, and Nando De~Freitas.
\newblock Taking the human out of the loop: A review of bayesian optimization.
\newblock \emph{Proceedings of the IEEE}, 104\penalty0 (1):\penalty0 148--175, 2015.

\bibitem[Shustin \& Avron(2021)Shustin and Avron]{shustin2021gauss}
Paz~Fink Shustin and Haim Avron.
\newblock Gauss-legendre features for gaussian process regression.
\newblock \emph{arXiv preprint arXiv:2101.01137}, 2021.

\bibitem[Srinivas et~al.(2010)Srinivas, Krause, Kakade, and Seeger]{srinivas2010gaussian}
Niranjan Srinivas, Andreas Krause, Sham Kakade, and Matthias Seeger.
\newblock Gaussian process optimization in the bandit setting: No regret and experimental design.
\newblock In Johannes F{\"u}rnkranz and Thorsten Joachims (eds.), \emph{Proceedings of the 27th International Conference on Machine Learning (ICML-10)}, pp.\  1015--1022, Haifa, Israel, June 2010. Omnipress.
\newblock URL \url{http://www.icml2010.org/papers/422.pdf}.

\bibitem[Titsias(2009)]{titsias2009variational}
Michalis Titsias.
\newblock Variational learning of inducing variables in sparse gaussian processes.
\newblock In \emph{Artificial intelligence and statistics}, pp.\  567--574. PMLR, 2009.

\bibitem[Wang et~al.(2019)Wang, Pleiss, Gardner, Tyree, Weinberger, and Wilson]{wang2019exact}
Ke~Wang, Geoff Pleiss, Jacob Gardner, Stephen Tyree, Kilian~Q Weinberger, and Andrew~Gordon Wilson.
\newblock Exact gaussian processes on a million data points.
\newblock \emph{Advances in Neural Information Processing Systems}, 32:\penalty0 14648--14659, 2019.

\bibitem[Wang et~al.(2017)Wang, Fang, and Liu]{wang2017stochastic}
Mengdi Wang, Ethan~X Fang, and Han Liu.
\newblock Stochastic compositional gradient descent: algorithms for minimizing compositions of expected-value functions.
\newblock \emph{Mathematical Programming}, 161\penalty0 (1-2):\penalty0 419--449, 2017.

\bibitem[Wilson \& Nickisch(2015)Wilson and Nickisch]{wilson2015kernel}
Andrew Wilson and Hannes Nickisch.
\newblock Kernel interpolation for scalable structured gaussian processes (kiss-gp).
\newblock In \emph{International Conference on Machine Learning}, pp.\  1775--1784. PMLR, 2015.

\bibitem[Wilson et~al.(2016{\natexlab{a}})Wilson, Hu, Salakhutdinov, and Xing]{wilson2016stochastic}
Andrew~G Wilson, Zhiting Hu, Russ~R Salakhutdinov, and Eric~P Xing.
\newblock Stochastic variational deep kernel learning.
\newblock \emph{Advances in Neural Information Processing Systems}, 29:\penalty0 2586--2594, 2016{\natexlab{a}}.

\bibitem[Wilson et~al.(2016{\natexlab{b}})Wilson, Hu, Salakhutdinov, and Xing]{wilson2016deep}
Andrew~Gordon Wilson, Zhiting Hu, Ruslan Salakhutdinov, and Eric~P Xing.
\newblock Deep kernel learning.
\newblock In \emph{Artificial intelligence and statistics}, pp.\  370--378. PMLR, 2016{\natexlab{b}}.

\bibitem[Yang et~al.(2015)Yang, Wilson, Smola, and Song]{yang2015carte}
Zichao Yang, Andrew Wilson, Alex Smola, and Le~Song.
\newblock A la carte--learning fast kernels.
\newblock In \emph{Artificial Intelligence and Statistics}, pp.\  1098--1106. PMLR, 2015.

\bibitem[Zhao \& Sun(2016)Zhao and Sun]{zhao2016variational}
Jing Zhao and Shiliang Sun.
\newblock Variational dependent multi-output gaussian process dynamical systems.
\newblock \emph{The Journal of Machine Learning Research}, 17\penalty0 (1):\penalty0 4134--4169, 2016.

\end{thebibliography}

\clearpage

\appendix

\section{Complexity Analysis}

We analyze the computational complexity for each one of the algorithms.
We ignore the complexity of for the computation of $\phi_{\alpha}\left(x_{i}\right)$
and $\frac{\phi_{\alpha}\left(x_{i}\right)}{\partial\alpha}$ as they
are the same for all the algorithms. We let $d$ be the dimension
of $\phi_{\alpha}\left(x_{i}\right)$ and $b$ be the batch size.

\subsection{Minimax}

\subsubsection*{Computational Complexity Forward}

\begin{tabular}{>{\centering}p{0.45\textwidth}>{\centering}p{0.5\textwidth}}
\textbf{Component} & \textbf{Complexity}\tabularnewline
\midrule
$g_{i}$ & $O\left(d\right)$\tabularnewline
$F_{\text{i}}$ & $O\left(d^{2}\right)$\tabularnewline
$h\left(A\right)$ & $O\left(d^{3}\right)$\tabularnewline
$\frac{\left\langle B,\frac{1}{n}A-F_{i}\left(\theta\right)\right\rangle }{\left\Vert A\right\Vert }$ & $O\left(d^{2}\right)$\tabularnewline
\midrule 
\textbf{Total} & $O\left(d^{3}\right)+O\left(bd^{2}\right)$\tabularnewline
\end{tabular}

\subsubsection*{Computational Complexity Backward}

\begin{tabular}{>{\centering}p{0.45\textwidth}>{\centering}p{0.5\textwidth}}
\textbf{Component} & \textbf{Complexity}\tabularnewline
\midrule 
$\frac{\partial g_{i}}{\partial w}$ & $O\left(d\right)$\tabularnewline
$\frac{\partial g_{i}}{\partial\phi_{i}}$ & $O\left(d\right)$\tabularnewline
$\frac{\partial h}{A}=A^{-1}$ & $O\left(d^{3}\right)$\tabularnewline
$\frac{\partial}{\partial B}$$\frac{\left\langle B,\frac{1}{n}A\right\rangle }{\left\Vert A\right\Vert }=\frac{\frac{1}{n}A}{\left\Vert A\right\Vert }$ & $O\left(d^{2}\right)$\tabularnewline
$\frac{\partial}{\partial A}$$\frac{\left\langle B,\frac{1}{n}A\right\rangle }{\left\Vert A\right\Vert }=\frac{\frac{1}{n}B\left\Vert A\right\Vert -\left\langle B,\frac{1}{n}A\right\rangle \frac{A}{\norm A}}{\left\Vert A\right\Vert ^{2}}$ & $O\left(d^{2}\right)$\tabularnewline
$\frac{\partial}{\partial\phi_{i}}$$\frac{\left\langle B,F_{i}\right\rangle }{\left\Vert A\right\Vert }=\frac{2B\phi_{i}}{\left\Vert A\right\Vert }$ & $O\left(d^{2}\right)$\tabularnewline
\midrule 
\textbf{Total} & $O\left(d^{3}\right)+O\left(bd^{2}\right)$\tabularnewline
\end{tabular}

\subsubsection*{Memory Complexity Forward}

\begin{tabular}{>{\centering}p{0.45\textwidth}>{\centering}p{0.5\textwidth}}
\textbf{Component} & \textbf{Complexity}\tabularnewline
\midrule
$g_{i}$ & $O\left(d\right)$\tabularnewline
$A,B,F_{i}$ & $O\left(d^{2}\right)$\tabularnewline
\midrule 
\textbf{Total} & $O\left(d^{2}\right)$\tabularnewline
\end{tabular}

\subsubsection*{Memory Complexity Backward}

\begin{tabular}{>{\centering}p{0.45\textwidth}>{\centering}p{0.5\textwidth}}
\textbf{Component} & \textbf{Complexity}\tabularnewline
\midrule
$\frac{\partial g_{i}}{\partial w}$ & $O\left(d\right)$\tabularnewline
$\frac{\partial g_{i}}{\partial\phi_{i}}$ & $O\left(d\right)$\tabularnewline
$\frac{\partial h}{A}=A^{-1}$ & $O\left(d^{2}\right)$\tabularnewline
$\frac{\partial}{\partial B}$$\frac{\left\langle B,\frac{1}{n}A\right\rangle }{\left\Vert A\right\Vert }=\frac{\frac{1}{n}A}{\left\Vert A\right\Vert }$ & $O\left(d^{2}\right)$\tabularnewline
$\frac{\partial}{\partial A}$$\frac{\left\langle B,\frac{1}{n}A\right\rangle }{\left\Vert A\right\Vert }=\frac{\frac{1}{n}B\left\Vert A\right\Vert -\left\langle B,\frac{1}{n}A\right\rangle \frac{A}{\norm A}}{\left\Vert A\right\Vert ^{2}}$ & $O\left(d^{2}\right)$\tabularnewline
$\frac{\partial}{\partial\phi_{i}}$$\frac{\left\langle B,F_{i}\right\rangle }{\left\Vert A\right\Vert }=\frac{2B\phi_{i}}{\left\Vert A\right\Vert }$ & $O\left(d\right)$\tabularnewline
\midrule 
\textbf{Total} & $O\left(d^{2}\right)+O\left(bd\right)$\tabularnewline
\end{tabular}

\subsection{SCGD}

\subsubsection*{Computational Complexity Forward}

\begin{tabular}{>{\centering}p{0.45\textwidth}>{\centering}p{0.5\textwidth}}
\textbf{Component} & \textbf{Complexity}\tabularnewline
\midrule
$g_{i}$ & $O\left(d\right)$\tabularnewline
$\tilde{F}^{-1}$ & $O\left(d^{3}\right)$\tabularnewline
$F_{\text{i}}$ & $O\left(d^{2}\right)$\tabularnewline
\midrule 
\textbf{Total} & $O\left(d^{3}\right)$\tabularnewline
\end{tabular}

\subsubsection*{Computational Complexity Backward}

\begin{tabular}{>{\centering}p{0.45\textwidth}>{\centering}p{0.5\textwidth}}
\textbf{Component} & \textbf{Complexity}\tabularnewline
\midrule
$\frac{\partial g_{i}}{\partial w}$ & $O\left(d\right)$\tabularnewline
$\frac{\partial g_{i}}{\partial\phi_{i}}$ & $O\left(d\right)$\tabularnewline
$\frac{\partial}{\partial\phi_{i}}\left\langle \tilde{F}_{t}^{-1},F_{i}\right\rangle =2\phi_{i}\tilde{F}_{t}^{-1}$ & $O\left(d^{2}\right)$\tabularnewline
\midrule 
\textbf{Total} & $O\left(bd^{2}\right)$\tabularnewline
\end{tabular}

\subsubsection*{Memory Complexity Forward}

\begin{tabular}{>{\centering}p{0.45\textwidth}>{\centering}p{0.5\textwidth}}
\textbf{Component} & \textbf{Complexity}\tabularnewline
\midrule
\foreignlanguage{american}{$g_{i}$} & \foreignlanguage{american}{$O\left(d\right)$}\tabularnewline
$F_{i}$ & $O\left(d^{2}\right)$\tabularnewline
$\tilde{F}$ & $O\left(d^{2}\right)$\tabularnewline
\midrule 
\textbf{Total} & $O\left(d^{2}\right)$\tabularnewline
\end{tabular}

\subsubsection*{Memory Complexity Backward}

\begin{tabular}{>{\centering}p{0.45\textwidth}>{\centering}p{0.5\textwidth}}
\multicolumn{2}{c}{\foreignlanguage{american}{}}\tabularnewline
\textbf{Component} & \textbf{Complexity}\tabularnewline
\midrule
$\frac{\partial g_{i}}{\partial w}$ & $O\left(d\right)$\tabularnewline
$\frac{\partial g_{i}}{\partial\phi_{i}}$ & $O\left(d\right)$\tabularnewline
$\frac{\partial}{\partial\phi_{i}}\left\langle \tilde{F}_{t}^{-1},F_{i}\right\rangle =2\phi_{i}\tilde{F}_{t}^{-1}$ & $O\left(d\right)$\tabularnewline
\midrule 
\textbf{Total} & $O\left(bd\right)$\tabularnewline
\end{tabular}

\subsection{BSGD}

\begin{multline*}
\frac{\partial}{\partial Z}\left[\y^{T}\left(K+\sigma^{2}I\right)^{-1}\y+\log\left|K+\sigma^{2}I\right|\right]\\
=-2\left(K+\sigma^{2}I\right)^{-1}\y\y^{T}\left(K+\sigma^{2}I\right)^{-1}Z\\
+2\left(K+\sigma^{2}I\right)^{-1}Z
\end{multline*}
Computational complexity: $O\left(b^{3}\right)+O\left(b^{2}d\right)$. 

Memory complexity: $O\left(b^{2}\right)+O\left(bd\right)$

\section{Missing Proofs}\label{sec:Missing-Proofs}
\begin{thm}
For all $\lambda>0$, $V\in\R^{n\times d}$ and $\b\in\R^{d}$ we
have that $\b^{T}\left(VV^{T}+\lambda I\right)^{-1}\b=\min_{\w}\frac{1}{\lambda}\left\Vert V\w-\b\right\Vert +\left\Vert \w\right\Vert ^{2}$
\end{thm}

\begin{proof}
The minimizer of $\frac{1}{\lambda}\left\Vert V\w-\b\right\Vert ^{2}+\left\Vert \w\right\Vert ^{2}$
is given by $\hat{\w}=\left(V^{T}V+\lambda I\right)^{-1}V^{T}\b$.
With this we can calculate
\begin{align*}
 & \min_{\w}\frac{1}{\lambda}\left\Vert V\w-\b\right\Vert ^{2}+\left\Vert \w\right\Vert ^{2}=\frac{1}{\lambda}\left\Vert V\hat{\w}-\b\right\Vert ^{2}+\left\Vert \hat{\w}\right\Vert ^{2}\\
 & =\frac{1}{\lambda}\left[\hat{\w}^{T}V^{T}V\hat{\w}-2\b^{T}V\hat{\w}+\y^{T}\y+\lambda\hat{\w}^{T}\hat{\w}\right]\\
 & =\frac{1}{\lambda}\left[\b^{T}\b-\left(2\b^{T}V\hat{\w}-\hat{\w}^{T}\left(V^{T}V+\lambda I\right)\hat{\w}\right)\right]\\
 & =\frac{1}{\lambda}\Big[\b^{T}\b-\Big(2\b^{T}V\left(V^{T}V+\lambda I\right)^{-1}V^{T}\b\\
 & -\b^{T}V\left(V^{T}V+\lambda I\right)^{-1}\left(V^{T}V+I\right)\left(V^{T}V+\lambda I\right)^{-1}V^{T}\Big)\b\Big]\\
 & =\frac{1}{\lambda}\b^{T}\left[I-V\left(V^{T}V+\lambda I\right)^{-1}V^{T}\right]\b\\
 & =\b^{T}\left(VV^{T}+\lambda I\right)^{-1}\b\ \ \ \text{(by Woodbury formula)}
\end{align*}
\selectlanguage{american}%
\end{proof}

\end{document}